\documentclass[12pt,authoryear,times,5p,twocolumn,url=false]{elsarticle}
\usepackage{epstopdf}

\usepackage{graphicx}
\usepackage{pst-all}
\usepackage{amssymb}
\usepackage{hyperref}
\usepackage{amsmath}
\usepackage{makeidx}
\usepackage{mathrsfs}
\usepackage{eurosym}
\usepackage{rotating}
\usepackage{multirow}
\usepackage{algorithm}
\usepackage{algorithmic}
\usepackage{wrapfig}
\usepackage{setspace}
\usepackage{fullpage}
\usepackage{amsthm}
\usepackage{resizegather}

\newcommand{\set}[1]{\left\{#1\right\}}

\newcommand{\meq} {\overset{!}{=}}
\newcommand{\qeq} {\overset{?}{=}}

\newcommand{\abs}[1]{\left\Vert#1\right\Vert}
\newtheorem{theorem}{Theorem}[section]
\newtheorem{proposition}[theorem]{Proposition} 
\theoremstyle{definition}
\newtheorem{example}{Example}[section]
\newtheorem{definition}{Definition}[section]

\journal{arXiv}

\begin{document}

\begin{frontmatter}
\title{Conditional probability generation methods for high reliability effects-based decision making}

\author[wg]{Wolfgang Garn\corref{cor1}}
\address[wg]{Surrey Business School, University of Surrey, Guildford, Surrey, GU2 7XH, United Kingdom}
\cortext[cor1]{Tel.:+44(0)1483682005;fax:+44(0)1483689511}
\ead{w.garn@surrey.ac.uk}

\author[pl]{Panos Louvieris}
\address[pl]{School of Information Systems, Brunel University, Uxbridge, UB8 3PH, United Kingdom}
\ead{panos.louvieris@brunel.ac.uk}

\begin{abstract}
	Decision making is often based on Bayesian networks. The building blocks for Bayesian networks are its conditional probability tables (CPTs).
These tables are obtained by parameter estimation methods, or they are elicited from subject matter experts (SME).
Some of these knowledge representations are insufficient approximations.
Using knowledge fusion of cause and effect observations lead to better predictive decisions.
We propose three new methods to generate CPTs, which even work when only soft evidence is provided.
The first two are novel ways of mapping conditional expectations to the probability space. 
The third is a column extraction method, which obtains CPTs from nonlinear functions such as the multinomial logistic regression.
Case studies on military effects and burnt forest desertification have demonstrated that so derived CPTs have highly reliable predictive power, including superiority over the CPTs obtained from SMEs. In this context, new quality measures for determining the goodness of a CPT and for comparing CPTs with each other have been introduced.
The predictive power and enhanced reliability of decision making based on the novel CPT generation methods presented in this paper have been confirmed and validated within the context of the case studies.
\end{abstract}


\end{frontmatter}

\section{Introduction}\label{sec:introduction}
Decision Making and Modeling in high pressure, fast moving, complex  environments is often confounded by the inability of the decision model to capture the requisite variety of the situation in a parsimonious manner. Novel techniques for CPT generation which address this capability gap are presented in this paper together with their concomitant case studies. These approaches are considered for application and evaluation herein.

Bayesian Networks have long been established in the literature as a useful tool for reasoning under uncertainty. Normally there are two stages involved in creating a Bayesian Network. The first stage is the derivation of a representative graph (topology).
\citet{Pearl1988}, \citet{Mengshoel20061137} and many more have provided fundamental work in this field. The second stage - which this paper focuses on deals with the challenge of obtaining probability distributions enabling the graph to be used for reasoning under uncertainty. 

Before introducing our novel CPT generation techniques and confirming their predictive power a review of current and past approaches used to derive CPTs are considered.

Methods for generating conditional probability tables (CPTs) are: noisy-max, noisy-or and noisy-and.
\citet{Pearl1986241} introduced and discussed the noisy-or which can be traced back to \citet{GOOD02011961} and \cite{GOOD05011961}. These approaches only work for binary variables.
\citet{diez93} showed a generalized noisy-or gate for multi-valued variables. 
\citet{Zagorecki2013} examined three test instances and found that noisy-max gates were a good fit for half of the given CPTs.
Another generalization of the noisy-or is the recursive noisy-or by \citet{DBLP:journals/tsmc/LemmerG04}. \citet{Vomlel2014} introduced noisy-thresholds in combination with a novel tensor representation of CPTs. These and other noisy-methods belong to class of noisy functional dependency methods. These operate on a parsimonious set of input data (in general directly proportional to the number of states) and require additional constraints (e.g. order, independence).
Techniques employed for CPT generation in order to support decision makers in our problem space must be flexible enough to accommodate a greater degree of variety than that imposed by the Noisy-Max method which imposes a sequence order constraint on the CPT output \cite{Li2011}. Such a restriction is prone to produce erroneous CPTs. 
Furthermore, these noisy-functional methods require the observation of hard evidence and corresponding ``single'' effect probabilities. However, often only soft evidence is available combined with uncertainty in the outcome. Noisy methods have not be designed for abundant or contradicting information.
Therefore, there is a requirement to develop methods that overcome the above mentioned limitations. In this paper we present techniques that solve the above mentioned issues.

CPTs can be determined using parameter estimation methods. They can be based on large data sets. \citet{Jensen2001} introduce several batch learning methods. On the other hand CPTs may adapt, whilst in operation by incrementally improving with each new case. Some of these methods are even capable of creating the Bayesian network structure.
\citet{Cooper92abayesian} present the construction of Bayesian Networks from databases. They showed ways to derive the conditional expectancy by means of frequency observations. In the subsequent analysis we will extend this conditional expectancy approach in a natural manner by the usage of regression based CPTs.
Popular parameter estimation techniques include the maximum likelihood estimation (MLE), Bayesian estimation and Expectation-Maximization (EM).
We will discuss MLE and EM in section \ref{MLE} and \ref{EM} respectively. 
These methods can be seen as having an objective subject to constraints.
\citet{Zhou2014} describe a constrained optimization approach, which reveals similar conceptual ideas to the ones used in this paper but realizing it with different methods.
In general it is tempting to refer to to statistical learning techniques in order to derive CPTs, but ignoring the fact that most of them cannot easily be mapped to the probability space. This paper helps in closing this gap in the body of knowledge.
Our new heuristics create CPTs based on conditional expectation. A comparison of such a CPT with the ones elicited by subject matter experts demonstrate their predictive power. 
The process of eliciting the information with minimum effort has been addressed by \citet{Bhattacharjya2014} by determining the order in which the CPT should be queried from a single expert. 
\citet{Xiang2007} investigated ways of finding CPTs by proposing a causal tree model.
Finding consensus between experts is a challenge as we observed in our military case study. This issue has also been investigated in \cite{Lopez-Cruz2014}, who developed Bayesian network to facilitate this task. Our research will offer another way to offer a combined view that may be used for consensus. Furthermore, our experiments indicate that a small number of cause and effect observations can be sufficient to determine a CPT.
New methods to compare CPTs with each other are introduced. Moreover, a methodology on the variance and goodness of CPTs is fully developed and presented. A case study on military effects and one on burnt forest desertification demonstrate the stability and accuracy of the derived CPT. In general the technique is especially useful when causes and effects have been measured, leading to a wide range of applications such as healthcare \cite{Velikova2014}.

In this section existing methods have been reviewed that derive conditional probability tables. Section \ref{sec:PPD_Bayesian_Network} introduces an operator that joins multiple probability vectors into a single one. This gives the opportunity to use matrix analysis to describe fundamentals of Bayesian networks.
In section \ref{sec:CPT Approximation} we will show novel approaches to generate CPTs.
New techniques for the goodness of conditional probability tables (CPTs) are proposed in section \ref{sec:Goodness of Models}.
In section \ref{sec:PPD_NN_Maria} and \ref{sec:PPD_Real_World_Application} we consolidate and validate the theory by using case studies. These case studies demonstrate the robustness of CPTs in their predictive power.

\section{Bayesian Network}\label{sec:PPD_Bayesian_Network}
In this section we will introduce Bayesian network topics which are fundamental for the generation of conditional probability tables. Furthermore we prove the equivalence between two particular Bayesian network structures.

\subsection{The Basics}
The fundamental difference of a Bayesian Network to other network structures is the usage of random variables (conditional probabilities) as ``weights" for the nodes. The conditional probability $P(Z|X)$ of event $Z$ occurring given that event $X$ has happened first is defined by $P(Z=z|X=x) := {P(Z=z,X=x)}/{P(X=x)}$.
We will relate this to a \emph{simple Bayesian network} with two nodes: effect (child) $Z$ and cause (parent) $X$ as illustrated in figure \ref{fig:SimpleBayesianNetwork} (a).
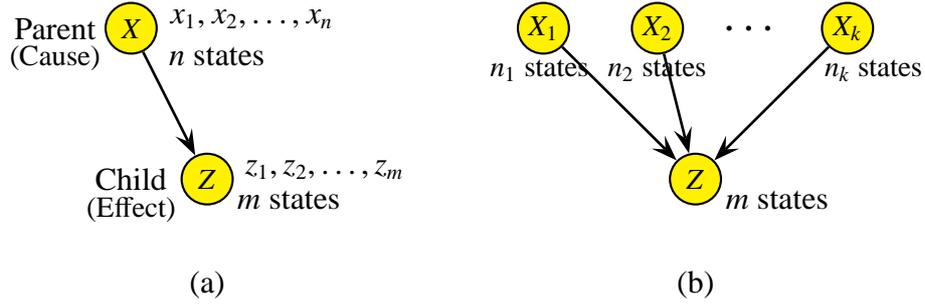
\begin{figure*}\centering
  \begin{tabular}{cc}
    \psset{unit=1cm,arrowscale=2, radius=0.35,fillstyle=solid, fillcolor=yellow}
    \begin{pspicture}(0,0)(6,4)
        \rput(1,3){Parent}
        \rput(1,2.6){\small{(Cause)}}
        \Cnode(2,3){node_X} \rput(2,3){\setlength{\fboxsep}{1pt}\colorbox{yellow}{\color{black} \small{$X$}}}
        \rput[bl](2.5,2.5){$n$ states}
        \rput[bl](2.5,3){$x_1,x_2,\dots,x_n$}

        \Cnode(3,1){node_Z} \rput(3,1){\setlength{\fboxsep}{1pt}\colorbox{yellow}{\color{black} \small{$Z$}}}
        \rput(2,1){Child}
        \rput(2,.6){\small{(Effect)}}
        \rput[bl](3.4,.6){$m$ states}
        \rput[bl](3.5,1){$z_1,z_2,\dots,z_m$}

        \ncline[linewidth=1pt]{->}{ node_X}{node_Z}

    \end{pspicture}
    &
      \psset{unit=1cm,arrowscale=2, radius=0.35,fillstyle=solid, fillcolor=yellow}
    \begin{pspicture}(0,0)(6,4)
        \Cnode(1  ,3){node_X1} \rput(1   ,3){\setlength{\fboxsep}{1pt}\colorbox{yellow}{\color{black} \small{$X_1$}}}
        \Cnode(2.5,3){node_X2} \rput(2.5 ,3){\setlength{\fboxsep}{1pt}\colorbox{yellow}{\color{black} \small{$X_2$}}}
                               \rput(3.75,3){\Large{$\dots$}}
        \Cnode(5  ,3){node_Xk} \rput(5   ,3){\setlength{\fboxsep}{1pt}\colorbox{yellow}{\color{black} \small{$X_k$}}}

        \Cnode(3  ,1){node_Z}  \rput(3  ,1){\setlength{\fboxsep}{1pt}\colorbox{yellow}{\color{black} \small{$Z$}}}
        \rput[bl](.3  ,2.3){\small{$n_1$ states}}
        \rput[bm](2.5 ,2.3){\small{$n_2$ states}}
        \rput[bl](4.7 ,2.3){\small{$n_k$ states}}
        \rput[bl](3.4,.6){$m$ states}

        \ncline[linewidth=1pt]{->}{ node_X1}{node_Z}
        \ncline[linewidth=1pt]{->}{ node_X2}{node_Z}
        \ncline[linewidth=1pt]{->}{ node_Xk}{node_Z}

    \end{pspicture}\\
    (a) & (b) \\
  \end{tabular}
  \caption{(a) Simple Bayesian Network, (b) converging Bayesian Network. \label{fig:SimpleBayesianNetwork}}
\end{figure*}

We introduce a vector-matrix notation for the probabilities which is necessary since the probability/statistic literature tends to mix single value and matrix notation.
The cause probabilities are $P(X)$ (which is a $n\times1$ vector) and the effect probabilities are $P(Z)$ (which is a $m\times1$ vector). Occasionally we will abbreviate the probability vectors as $x$ and $z$ (rather than reserving these letters to symbolize the state of the random variable). For now we assume that the conditional probability table (CPT) for the Bayesian node $Z$ is known. The CPT $P(Z|X)$ is abbreviated - with the new notation, $z|x$, which is the $m\times n$ matrix.
In \citet{Pearl1988} the notation $M_{z|x}$ is used for the CPT and called conditional probability matrix or link matrix. To make rows and columns explicit:
\begin{gather}
 P(Z|X) = \left(
             \begin{array}{ccc}
               P(Z=z_1|X=x_1) & \dots & P(Z=z_1|X=x_n) \\
               \vdots & \ddots & \vdots \\
               P(Z=z_m|X=x_1) & \dots & P(Z=z_m|X=x_n) \\
             \end{array}
           \right).
\end{gather}
Typical possible notations for the conditional probability table are: $P(Z|X) = M_{z|x} = C$.

The joint probability matrix is defined in a similar way $\widehat{z,x} := P(Z,X)$, which is also a $m\times n$ matrix. It can be derived from the conditional probability matrix by:
\begin{equation}
     P(Z,X) = P(Z|X) \cdot (P(X)~e_n),
\end{equation}
where the point $\cdot$ represents element by element multiplication and $e_n$ is the a row vector $n$ ones. 
A more efficient notation would be $\widehat{z,x} := z|x \cdot (x e_n)$.

The effect probabilities in matrix notation are computed by:
\begin{equation}
    P(Z) = P(Z|X) P(X),
\end{equation}
which follows from the definition of the conditional probability and the marginalisation of the joint probability table. 
Again this could be more efficient expressed as $z = z|x~x$.

The new notation provides a simplified and more efficient representation of the probability matrix formulation and proves particularly useful when it comes to representing complex Bayesian Network computations. Here, we would like to draw the readers attention to \citet{Vomlel2014} representation of CPTs as tensors. 

\subsection{Cause Computation}
Given the effects and conditional probabilities it is possible to determine the causes using \emph{Bayes Rules}:
\begin{gather}\label{eq:Bayes-rule}
P(Z=z|X=x)P(X=x) = P(X=x|Z=z)P(Z=z).
\end{gather}
 From Bayes rule we obtain $P(X=x|Z=z) = {P(X=x,Z=z)}/{P(Z=z)}$, which reverses the effect computation and leads to a conditional probability matrix. Let us put this into matrix formulation:
\begin{equation}
    x = \widehat{x,z} \div (e_m z') \Leftrightarrow P(X) = P(X,Z) \div (e_m P(Z)'),
\end{equation}
where $e_m$ is the vector consisting out of $m$ ones and $\div$ is the element-wise division operation. Thus we have a method to determine the cause probabilities given the effects. This allows us to deduce the causes given the effects.
Note that $z = Cx \Leftrightarrow (C C')^{-1}C'z=x$ cannot be used because it leaves the probability space.


\subsection{Converging Network}
\begin{definition}
A \emph{Bayesian network} $\mathfrak{B}$ is a simple directed acyclic graph, which consists out of a sequence of nodes $N$ with an associated sequence of weights $W$ and a sequence of arcs $A$.
\end{definition}
The \emph{weights} in a discrete Bayesian network are conditional probability tables or probability vectors. Note that a probability vector can be interpreted as a special case of a conditional probability table.
The literature \cite{AI_Russell:113848} distinguish between three types of Bayesian Networks: discrete, continuous and hybrid Bayesian Networks.
In this document we will deal only with discrete BN and will simply call them BN.

\begin{definition}
Let $\mathfrak{B}_1$ be a Bayesian Network (BN) illustrated in figure \ref{fig:SimpleBayesianNetwork}(b) and defined by :
\begin{equation}\begin{array}{l}
  N:=\langle Z,X_1,X_2,\dots,X_k\rangle, \\
  W:=\langle P(Z|X_1,\dots,X_k),P(X_1),P(X_2),\dots,P(X_k) \rangle, \\
  A:=\langle (X_1,Z), (X_2,Z),\dots, (X_k,Z) \rangle ~\mbox{and}~k\geq1.
\end{array}\end{equation}
Any Bayesian networks of the above structure will be called a \emph{converging Bayesian network}.
\end{definition}
A representation of multiple nodes as a single node is known as a \emph{cluster node} \cite{Pearl1988}.

\begin{definition}
A BN $\mathfrak{B}_2$ (shown in figure \ref{fig:SimpleBayesianNetwork}(a)) is defined by :
\begin{equation} N:=\langle Z,X\rangle,  W:=\langle P(Z|X),P(X)\rangle,  A:=\langle (X,Z) \rangle\end{equation} and will be referred to as \emph{simple Bayesian network}.
\end{definition}

Any Bayesian Networks of the structure shown in figure \ref{fig:SimpleBayesianNetwork}(b) can be transformed to a Bayesian network having the structure displayed in \ref{fig:SimpleBayesianNetwork}(a). That means the parent $X$ in figure \ref{fig:SimpleBayesianNetwork}(a) has $n=n_1+n_2+\dots+n_k$ states.

\begin{theorem}\label{theorem:BN_Trans}
There is a bijective mapping between a converging BN $\mathfrak{B}_1$ and a simple BN $\mathfrak{B}_2$, such that the effect probabilities are equal.
\end{theorem}
\begin{proof}
If $n=1$ the mapping is trivial. If $n=2$ then the transpose of the outer product $Y = (x_1 x_2')'$ is a $n_2 \times n_1$ matrix. Let $Y_:$ designate the transformation of the $Y$ matrix into the $y$ column vector by concatenating all column vectors of $Y$. $y$ represents $P(X)$ for $\mathfrak{B}_2$.
Conversely $P(X)$ can be split into $x_1$ and $x_2$, if the number of states are known by considering the matrix $Y$. The vectors are obtained by realising that the outer product constitutes a joint probability table. Applying Bayes rule (eq. \ref{eq:Bayes-rule}) leads to the reverse CPT.

If $n \in \mathbb{N}$, then $P(X)$ can be derived by algorithm \ref{alg:nNodes2One}. The other direction requires the knowledge of the number of states and that we can reverse the CPT multiplication. We will elaborate on this, let $C=z|x_1x_2\dots x_n=P(Z|X_1,\dots,X_n)$. The effect probabilities can be computed by $z=C x$. Note that $x$ cannot be determined by $x= (C'C)^{-1} C'z$ even when $C'C$ is nonsingular ($\det{(C'C)}\neq0$), because this would violate probability axioms. Hence, we have to determine the reverse of $C$ namely $x|z$ by using Bayes' rule (see equation \ref{eq:Bayes-rule}). Knowledge of the number of states allows us to create a joint matrix.  Marginalisation of the joint matrix leads to the probability vectors $x_1,x_2,\dots,x_n$. Thus we have shown the computation of the probability product. Hence, the computed effect probabilities of $\mathfrak{B}_1$ and  $\mathfrak{B}_2$ are equal.
\end{proof}

\subsection{Combine Operator}
Important in the above proof was the usage of algorithm \ref{alg:nNodes2One}, which combines multiple independent probability vectors into a single one; we define this as the \emph{combine operator}
$\circlearrowleft: [0,1]^{n_1} \times \dots \times [0,1]^{n_k} \rightarrow [0,1]^{n_1 \times \dots \times n_k}$.
\begin{algorithm}[h!]
    \caption{Multiple parents to single parent.
    (Operator $\circlearrowleft$)}
    \label{alg:nNodes2One}
    \begin{algorithmic}[1]
        \REQUIRE $x_1,\dots,x_n$ \dots parallel causes (column vectors)
        \ENSURE $x$ \dots product cause node from parallel causes

        \medskip
        \STATE $x := x_1$  assign first cause
        \FOR{$k=2$ to $n$ (for all remaining parallel causes)}
            \STATE $Y := (x x_k')'$  compute outer product and transpose
            \STATE $x := Y_:$ transform matrix into one single column vector by concatenating all columns
        \ENDFOR
    \end{algorithmic}
\end{algorithm}
Note that algorithm \ref{alg:nNodes2One} has a repeated application of the outer product and a matrix to vector transformation.
Just like with any other operator based on multiplication it is only possible in special cases to reverse the combine operation in a unique way.

\subsection{Evidence}
Evidence may be given as a sequence of probability ratios. For instance assume that the odds are 7 to 5, which is transformed into the probability vector $(.58~.42)$. This type of evidence will be called \emph{soft evidence}. It is also common in the BN literature to express evidence by stating that a node has assumed a certain state, which we will call \emph{hard evidence}. This is equivalent to having a 100\% probability in one state. We will call such a node an \emph{evidence node}. A converging network with all parents having hard evidence leads to the selection of exactly one column from the CPT.
\begin{proposition}\label{prop:CPT-column}
The effect probabilities are identical to one unique column of the CPT if all parents are evidence nodes (i.e. each parent has exactly one 100\% state) in a converging subnet.
\end{proposition}
This proposition is significant for the relationship between multinomial logistic regression and CPTs, which will be discussed later.

 Of course a combination of multiple evidence nodes is also possible which leads to selection of certain columns and their summation. This linear combination may be used to compute effect probabilities in a more efficient way.

\subsection{Joint Probability using the Combine Operator}
In general the joint probability of a converging network as shown in figure \ref{fig:SimpleBayesianNetwork} (b) is computed by:
\begin{equation}
    P(X_1, X_2, \dots, X_n, Z) = P(Z|X_1,\dots X_n) \prod_{k=1}^n  P(X_k).
\end{equation}
This has to be done for each state configuration. Of course we could apply the combine operator ("vectorising" outer product), which gives us the matrix representation:
\begin{equation}
    \widehat{z, x_1,x_2,\dots, x_n} = z|x_1x_2\dots x_n . (e \circlearrowleft_{k=1}^n x_k,)
\end{equation}
where $e$ is the ones vector needed for the point-wise multiplication with the combinational product of all probability vectors.

\subsection{Summary}
In this section we have proven the computational identity of converging and simple BN.
We have proposed a novel probability combination operator $\circlearrowleft$, which transforms probabilities of multiple parents into a single probability vector. This new operator integrates well with matrix notation. Furthermore we established an important new evidence proposition, which will be used to extract CPTs out of linear and nonlinear functions.

\section{CPT Generation}\label{sec:CPT Approximation}
\subsection{Introduction}
    In the beginning of the paper techniques to generate CPTs were outlined and their deficiencies were discussed.
		Here, the MLE and EM method will be revisited and explanations for their failings will be given. These techniques are more commonly placed in the area of parameter learning methods (see \cite{Koller2009} for details).  
    Then this section introduces new methods that can generate CPTs.
    They are based on the assumption that causes and effects can be observed.
    Even if they cannot be observed it is still better and easier for a subject matter expert to provide
    the causes and effects and use the techniques presented herein. Measures for goodness of these models (section \ref{sec:Goodness of Models}) will show that little information is required to generate CPTs with the new methods.

\subsection{Maximum Likelihood Estimation}\label{MLE}
A classic approach to obtain CPTs is to use the Maximum Likelihood Estimation (MLE):
\begin{equation}
	\hat{\theta} = \arg \max_\theta L(M_\theta|D),
\end{equation}
where $\theta$ are the parameters, $M_\theta$ the models, $D$ the given data and $L(M_\theta|D)=\prod_{d \in D} P(d|M)$. The data for a CPT is generally assumed to state the cause and effect state (as hard evidence). Hence relative frequencies can be derived, for instance:
\begin{equation}
	P(z_i|x_j)= \frac{N(z_i,x_j)}{N(x_j)}.
\end{equation}
$N(x_j)$ represents the number of occurrences of case $x_j$ in the data, and $N(z_i,x_j)$ the simultaneous occurrence of state $z_i$ and $x_j$. A few problems can be observed. 
(1) No occurrences of case $x_j$ are recorded; 
(2) $N(z_i,x_j)$ is zero (i.e. impossibility of $P(z_i|x_j)$), or 
(3) the case information was provided in a fuzzy way (soft evidence, e.g. certainty of a case $x_j$ is $d(x_j)=.7$). The fuzzy specification can be transformed into number measures by rounding $\lfloor d(x_j) \rceil$, reducing accuracy further. 
The main issue observed in the case studies was the incompleteness of data, which meant that the MLE could not be used to obtain meaningful CPTs.

\subsection{Expectation Maximization}\label{EM}
The MLE showed issues arising with incomplete data. These can be overcome using the expectation maximization algorithm. The aim is to find parameters $\theta$ such that the likelihood of the model is maximized. Initially the model and its parameters $\theta_t, t=0$ have to be assumed. This is followed by the ``Expectation'' step. Here the responsibilities are computed, which are the expected counts for the CPT:
\begin{equation}
	  E(N(z_i,x_j) | D) = \sum_{d \in D} P(z_i,x_j | d,\theta_t).
\end{equation}
The maximization step determines a new estimate for $\theta_{t}, t\leftarrow t+1$:
\begin{equation}
	\theta_{t} = \frac{E( N(z_i, x_j) | D )}{\sum_{j=1}^m E( N(z_j, x_j | D) )} 
\end{equation}
The expectation and maximization step are continued until the algorithm has converged, i.e. 
\begin{equation}
	|\ln P(D|\theta_{t+1}) - \ln P(D|\theta_t)| \leq \epsilon.
\end{equation}
One of the main issues with this procedure is the choice of the initial parameters, which will influence the local (or global) maximum achieved. One should be aware that this is a local search heuristic. That means embedding it into multi-start procedures or meta heuristics can improve the  solution quality. However, this adds to the run-time of this computationally expensive method. 
Another disadvantage - when using this method in its ``classic'' form - is that it uses occurrences (hard evidence) rather than accommodating soft evidence.
The advantage of the EM method is the iterative usage of conditional expectation. 

We will now begin to introduce new methods that use the benefits of MLE and EM. These are derived from the familiar conditional expectation (regression), but have to stay within the probability space.
  
\subsection{Conditional Mean Basis}\label{sec:PPD_Regression}
Assume that causes $x$ and effects $z$ were $k$ times observed. We will show how a CPT basis can be estimated from these observations. This is the basis for probability boundary limitation method and the probability potential surge method introduced in the subsequent sections to derive the real CPT, i.e. a CPT basis is a matrix which may still violate the probability constraints explained below.
The conditional mean function $E(z|x)$ is also called regression of $z$ on $x$.
An estimator of $E(z|x):=\sum_z z f(z|x)$ will be used to obtain the CPT.
where $f(.)$ is the probability density function (continuous case) or probability distribution (discrete case). $f(.)$ has to follow Kolmogorov's axioms.
General discussions of multiple regressions can be found in \cite{Hastie01} and \cite{Greene00}.
The application of the conditional mean to approximate the conditional probability table was used in \cite{Cooper92abayesian} in relation to Bayesian Networks. Roughly speaking they used frequencies to obtain the probabilities.

The new approach will need the following preparation.
Given are $k$ observations of the effects $Z=(z_1,\dots,z_k)'$ and causes $X=(x_1,\dots,x_k)'$. That means $Z$ can be represented as a $k\times m$ matrix and $X$ as a $k\times n$ matrix. The objective is to determine the CPT basis $B:=(b_{ij}):=z|x'$, such that the squared sum $S(B):=(Z-XB)'(Z-XB)$ obtains a minimum:
\begin{equation}
    B^*= \min_{B\in \mathbb{B}} \set{tr((Z-XB)'(Z-XB))},
\end{equation}
where $B$ transposed is the $m\times n$ matrix with elements in the interval $[0,1]$:
\begin{equation}\label{eq:prob_constraint_1} b_{ij} \in [0,1] ~ \forall i,j. \end{equation}
In order to satisfy Kolmogorov's axioms another constraint on $B$ must be fulfilled. The row sum of $B$ must add up to one:
\begin{equation}\label{eq:prob_constraint_2}
    \sum_{j=1}^m b_{ij} \meq 1 ~ \forall i \in \set{1,2,\dots,m}. 
\end{equation} 
$\mathbb{B}$ is the set of all matrices fulfilling the above mentioned two constraints. Note that the symbol $\meq$ denotes ``must equal".
Other objective functions to determine $B^*$ may be appropriate depending on the nature of the underlying problem.
However, we have used the least square method because of its characteristics and popularity.
We will limit ourselves to deriving a CPT basis using matrix calculus.
We obtain the first derivative of $S(B)$ by:
\begin{equation}
    \frac{\partial S(B)}{\partial B} = -2X'(Z-XB).
\end{equation}
Setting the first derivative to zero will give us the optimum:
\begin{equation} -2X'(Z-XB) = 2X'Z-2X'XB = 0,\end{equation}
under the assumption that $X$ has full column rank, which guarantees that $X'X$ is positive definite.
Thus, an optimal CPT basis $B^{*}$ is:
\begin{equation}\label{eq:CPT-Basis}
    \overbrace{B^* }^{k \times m} = \overbrace{(X'X)^{-1}}^{n\times n} \overbrace{X'}^{n \times k} \overbrace{Z}^{k \times m}.
\end{equation}

An alternative way of obtaining the above result is shown in the footnote
\footnote{Equation (\ref{eq:CPT-Basis}) could have been derived in a different way by looking for a $B^*$ which minimises the error matrix $E$ in $Z=XB^*+E$. Thus $E$ must be orthogonal to $XB$. That means:
$XB^* \bot E \Leftrightarrow (XB^*)'E=0 \Leftrightarrow (XB^*)'(Z-XB^*)=0 \Leftrightarrow {B^*}' (X'Z - X'XB^*)=0 \Rightarrow B^*=(X'X)^{-1}X'Z$.}.
Note that it cannot be guaranteed that the elements in $B^*$ fulfill the two constraints (\ref{eq:prob_constraint_1}) and (\ref{eq:prob_constraint_2}), because of the matrix inversion and multiplication. The case study given in section \ref{sec:PPD_Real_World_Application} supports this statement.

In order to enforce a CPT basis to become a proper CPT we propose the \emph{probability boundary limitation} and \emph{probability potential surge methods}.

\subsection{Probability boundary limitation method}\label{sec:PPD_Enforce_CPT}
If the value that is supposed to represent the probability is out of range it will be set to its closest limit. That means if $x_i<0$ then $x_i$ is set to zero and if $x_i>1$ it is assigned the value one, i.e.
\begin{equation}
    \underline{\overline{x_i}}_0^1 := \min \set{ \max \set{x_i,0}, 1} ~~~\forall i.
\end{equation}

Afterwards the column vector $\tilde{x}=\left( \underline{\overline{x_i}}_0^1 \right)$ is normalised by $x=\frac{\tilde{x}}{\sum_i \tilde{x_i}}$.
An exception which could happen is that all approximated ``probabilities" are smaller than zero, in this case one could make the approximation by shifting.
A crude solution, if all values are zero or one, is obtained by assign a uniform probability to $x$.

\subsection{Probability potential surge method}
The $\tilde{x_i}$ values of a column undergo a translation into the positive range by adding the negative of the lowest $\tilde{x_i}$ value. For instance $\tilde{x}=(-.2~-.3~.4)'$ then the lowest value is $m=-.3$ and the surge in potential leads to $\bar{x} = (.1~.0~.7)'$, which must be normalised by $\alpha^{-1} := \bar{x}'e_3=.1+0+.7=.8$. Below is the formal expression of the probability potential surge method:
\begin{equation}
    \exists m \in x: m<0 ~~~  \alpha (\underbrace{\tilde{x} - e_n \min{\tilde{x}}}_{\bar{x}}),
\end{equation}
where $\alpha$ is the normalisation coefficient ($\alpha^{-1}=\sum_i \bar{x_i} = \bar{x}' e_n$).

In section \ref{sec:PPD_Real_World_Application} we will apply these heuristics to case studies.

\cite{Berry:1993} noted that dichotomous outcomes violate the assumption of linear relationship of the variables in general. However, a logarithmic transformation preserves the non-linearity as shown in \cite{Berry&Feldman:1985}. This motivates that we will introduce the logistic regression and examine its applicability.





\subsection{Probability base vector extraction method}\label{sec:PPD_Logistic_Regression}
In this section we show a new method to extract a CPT from a logistic regression. However, the method can be applied to any function operating on $[0,1]$.

The multinominal distribution is given by:
\begin{equation}\label{eq:LL-prob}
    \begin{array}{ccl}
      P(Z=z_k) & = & e^{y_k}(1+ \sum_{k=1}^n e^{y_k})^{-1}, ~~ \forall k<n;\\
      P(Z=z_n) & = & (1+\sum_{k=1}^n e^{y_n})^{-1},
    \end{array}
\end{equation}
where $y_k= b_{k0} + \sum_{j=1}^n b_{kj} x_j = b_k [1; x']$. The logistic model is a special case of the generalised linear model, explained in \cite{McCu:Neld:1989}.
The effect described by its probabilities $P(Z)$ (dependent variable) is computed given cause probabilities (independent variables).
The likelihood is $l(b) := \prod_{k=1}^n P(Z=z_k|X=x)$, where $b$ describes the regression parameters. We will abbreviate $P(Z=z_k|X=x)$ by $p_k(x,b)$ to emphasise the dependence on $b$. This is equivalent to the logarithmic likelihood:
\begin{equation}\label{eq:LL-general}
    L(b) := \log l(b) = \sum_{k=1}^n \log p_k(x,b).
\end{equation}
Let us assume that $p_k(x,b) = p(x,b)^{c_k} (1-p(x,b))^{1-c_k}~~\forall k$, where $c_k$ describes the class of the $k$th observation and its case weight. Now we can write (\ref{eq:LL-general}) as
\begin{equation}\label{eq:LL}
   L(b) = \sum_{k=1}^n c_k \log p_k(x,b) + (1-c_k) \log (1-p_k(x,b)).
\end{equation}
To find the maximum probability of (\ref{eq:LL}) we take the first derivative and require that it must be zero:
\begin{equation}\label{eq:LL-derivative}
   \frac{\partial L(b)}{\partial b} = \sum_{k=1}^n x_k (y_k - p(x_k,b)) \meq 0.
\end{equation}
Usually equation (\ref{eq:LL-derivative}) is solved using the Newton-Raphson algorithm in order to obtain $b$.
This gives us a convenient way of computing the effect probabilities via equation \ref{eq:LL-prob}. So far this has been ``classic" multivariate data analysis which is explained in \cite{Hair_1998}, \cite{Hastie01} and \cite{Field:2009}.

Next we will introduce the new probability base vector extraction method and show that logistic regression can be used to determine the conditional probability table. Recall that in section \ref{sec:PPD_Bayesian_Network} we have shown that any column of a CPT matrix can be obtained by giving hard evidence (proposition \ref{prop:CPT-column}). That means a canonical cause matrix consisting entirely out of hard evidence determines the CPT. Thus the elements of the CPT are determined by:
\begin{equation}\label{eq:LL-CPT}
    \begin{array}{ccl}
      P(Z=z_k|X=\dot{e}_i) & = & \frac{e^{b_{k0}+ b_{ki}}}{1+ e^{b_{k0}} \sum_{j=1}^n e^{b_{kj}}}, ~~ \forall k<m, i\leq n;\\
      P(Z=z_m|X=\dot{e}_i) & = & \frac{1}{1+e^{b_{n0}} \sum_{j=1}^n e^{b_{nj}}}, ~~ \forall i\leq n,
    \end{array}
\end{equation}
where $\dot{e}_i$ is the hard evidence vector, i.e. element $i$ is one and the other elements in the vector are zero. This means that $y_k$ obtains a very simple form $y_k = b_{k0}+ b_{ki}$, which is reflected in equation (\ref{eq:LL-CPT}). \cite{Rijmen2008} has discussed this method in more detail. 
This method will be applied to the ``burnt forest" case study.

What are the benefits of using a logistic regression CPT rather than the logistic regression result directly? In a standard Bayesian Network we can use CPTs which are derived from logistic regression but not the logit function itself.
Furthermore the computation of the effect probabilities using the CPT is more efficient than using the logit function.

The here proposed extension - the probability base vector extraction method - of Logistic Regression provides a ``bridge" to Bayesian Networks.

\subsection{Summary and Concluding Remarks}
  We have introduced several novel methods to generate CPTs. These methods have extended existing popular techniques. The ease of applying the introduced extensions to the existing techniques should help to make the new methods valuable assets. Two of the CPT newly introduced generation methods were based on modifying the conditional mean through potential surge and boundary limitation heuristics. The third - column extraction method - demonstrated how to create a CPT from a non-linear function (logistic regression).
  Here our proposition \ref{prop:CPT-column} played an important role to extract an approximate CPT.

  The case studies in later sections will be used to evaluate the goodness of these models. A new set of measures is introduced in the following section, which we argue are simple, more intuitive and give the possibility of comparing the quality of CPT approximation models. 
\section{Goodness of Models} \label{sec:Goodness of Models}
There are many classical assessment methods of the goodness for regression (e.g. $R^2$) and logistic regression (e.g. $R$-statistic).
Often the Kullback-Leibler (KL) divergence measure is found in the literature (e.g. \cite{Zhou2014,Zagorecki2013}) for CPT comparisons, 
which was originally developed to measure the relative entropy in information. 
It measures the divergence of $\tilde{C}$ from $C$ by using $\sum_i x_i \sum_j c_{ij} \ln \frac{c_{ij}}{\tilde{c}_{ij}}$, where $C=P(Z|X)$ is the correct CPT and $\tilde{C}$ the approximated CPT (here $x$ is $P(X)$).
However, the logarithm is difficult to explain in the context of conditional probabilities. Some authors have used the Euclidean distance instead, i.e. the distance measure between CPTs is: $\sum_i x_i \sum_j (c_{ij}-\tilde{c}_{ij})^2$.

This section introduce measures that determine the goodness of effect estimates and approximated CPTs given training and test data. 

\subsection{Effect comparison}
There are four measures we are interested in:
diagnostic error ($d$),
total average shift error ($\bar{s}$),
mean state error ($s_j$), and
absolute effect observation error ($\delta_i$).\footnote{If both components are estimates then ``deviation'' is a better term than ``error''.}
A good illustration of these errors can be found in figure \ref{fig:MPGIS-LogitCPT-errors-test-data}. In the previous sections we demonstrated methods to determine an approximate CPT $\tilde{C}$ given causes $X$ ($k\times n$ matrix) and effects $Z$ ($k\times m$ matrix) training data. The quality of the $CPT$ is evaluated by using test data $X_t$ and $Z_t$. Of course this assumes that we trust the test data to be meaningful. The above four measures are the result of comparing $Z_t$ and $\tilde{Z_t}:=X_t C$ with each other.

\begin{definition}
The \emph{absolute effect observation error} compares each effect observation with the corresponding approximation. Let $z_i \in Z_t$ be the $i^{th}$ effect observation. That means $z_i:=(z_{i1},\dots,z_{im})$ is a probability vector consisting out of $m$ states. The same applies to the approximated effect $\tilde{z}_i \in \tilde{Z_t}$. Thus we can describe the absolute effect error $\delta_i$ for the $i^{th}$ observation by:
\begin{equation}
    \delta_i := \frac{1}{2}\sum_{j=1}^m |z_{ij} - \tilde{z}_{ij}|.
\end{equation}
\end{definition}

This measure adds up all deviations and thus gives us the total shift error of an approximated effect. We observe that $0 \overset{!}{\leq } \delta_i \overset{!}{\leq } 1$. If the probability vector has a highly likely state then the error can be more significant. An effect error can be interpreted as the likeliness of determining the wrong state. $1-\delta_i$ is the goodness of the effect approximation.

Of course the average absolute error is also of interest:
\begin{equation}
    \bar{\delta} := \frac{1}{k} \sum_{i=1}^k \delta_i.
\end{equation}

\begin{definition}
Comparing each state $j$ across all $k$ samples defines the \emph{mean state error}. More precisely:
\begin{equation}
    s_j := \frac{1}{k}\sum_{i=1}^k |z_{ij} - \tilde{z}_{ij}|.
\end{equation}
\end{definition}
This is the likelihood of a certain state being wrong. $1-s_j$ is mean goodness of a state.

Now we are in the position to summarise the goodness of all states. We determine the average of the states being wrong by:
\begin{equation}
    \bar{s} := \frac{1}{m}\sum_{j=1}^m s_j.
\end{equation}
The average goodness of all states is thus $1-\bar{s}$. This is a ``dangerous" measure, because it depends on the number of states. The more states an effect has the more a wrong impression of goodness is perceived.

The second possible interpretation of $\bar{s}$ focuses on the shifting of probabilities between the states.
\begin{definition}
Thus, $\bar{s}$ will be also called \emph{total average shift error}.
\end{definition}

\begin{definition}
The \emph{diagnostic goodness}
\footnote{The diagnostic goodness will be also called \emph{diagnostic power}.}
 $g$ counts the number of agreeing effect maximum states and sets them in proportion to the total number of observations:
\begin{equation}
   g := \frac{1}{k} \sum_{i=1}^k \left[\underset{j \in \set{1,\dots m}}{\mbox{argmax}}~ z_{ij} \qeq \underset{j \in \set{1,\dots m}}{\mbox{argmax}}~ \tilde{z}_{ij} \right].
\end{equation}
\end{definition}
Note that the symbol $\qeq$ denotes the boolean operator (query equality) that gives the value one if the expression on the left is equal to the expression on the right side and otherwise zero.

In the above equation we made use of the Iverson brackets.
\footnote{\citet{Iverson_1962} introduced these brackets for true-or-false statements $[S] := \begin{cases} 1 & \mbox{if } S \mbox{ is true;} \\ 0 & \mbox{otherwise.} \end{cases}$. For instance the Kronecker delta is defined by $ \delta_{ij} := [ i \qeq j ] $.}

\begin{example}
Assume that $z_{1}=(.3~.7), z_{2}=(.4~.6), z_{3}=(.9~.1)$ and $\tilde{z}_{1}=(.2~.8), \tilde{z}_{2}=(.5~.5), \tilde{z}_{3}=(.7~.3)$. Thus the diagnostic goodness is determined by: $g:=\frac{1}{3}([2\qeq 2] + [2 \qeq \set{1,2}] + [2 \qeq 1])=\frac{1}{3}(1+0+0)=\frac{1}{3}$. Hence, the diagnostic error is $\frac{2}{3}$.
\end{example}

\begin{definition}
The \emph{diagnostic error} $d$ is:
\begin{equation}
    d := 1 - g
\end{equation}
\end{definition}

\subsection{CPT comparison}

The KL and Euclidean method can be used to compare CPTs. Here, we will introduce a new method. We assume that one CPT $C=(c_{ij})=P(Z|X)$ is ``correct" and the other one is an approximation $\tilde{C}=(\tilde{c}_{ij})$. It is tempting to introduce a relative error measure. However, it can be easily shown that these are inappropriate for probability measures.
In the previous section we introduced the absolute effect observation error. This measure is based on the observation that probabilities are ``shifted" between states. We apply this principle to the CPT to obtain the

\begin{definition} \emph{CPT shift error}
\begin{equation}
    \bar{\delta} := \frac{1}{n} \sum_{j=1}^n (\frac{1}{2} \sum_{i=1}^m \abs{c_{ij} - \tilde{c}_{ij}})
              = \frac{1}{2n} \sum_{j=1}^n \sum_{i=1}^m \abs{c_{ij} - \tilde{c}_{ij}}.
\end{equation}
\end{definition}
This was previously known as the average of the absolute effect observation error. As before completely distinct CPTs are identified.
\begin{example}[Distinct CPTs]
Given are two distinct CPTS $C$ and $\tilde{C}$.
\begin{equation}C=
		\begin{pmatrix}
       1 &        0 &        0 \\
       0 &        1 &        0 \\
       0 &        0 &        1 \\
       0 &        0 &        0
    \end{pmatrix}
		,~~
\tilde{C} =     
	\begin{pmatrix}
       0 &        0 &        1 \\
       1 &        0 &        0 \\
       0 &        1 &        0 \\
       0 &        0 &        0
    \end{pmatrix}
\end{equation}
The CPT shift error is $\bar{\delta} = 100\%$.
\end{example}

\section{Case Study (1) - Burnt forest desertification risk}\label{sec:PPD_NN_Maria}
A Bayesian network to determine the risk of burnt forest desertification was introduced in \citet{BN_NeuralNetworks}.
The main purpose of this paper was to obtain the correspondence between Bayesian and Neural Networks. As well as discussing this link a case study about forest desertification was given.

Our objective in this section is to make use of the given Bayesian Network and the corresponding test data. We generate from this data several CPTs applying the probability boundary limitation method, probability potential surge method and the probability base vector extraction method. The goodness will be tested in respect of the CPTs predictive power and its total average shift error.

\subsection{Bayesian Network and Data}
Regeneration of Mediterranean forests is usually achieved between two to five years. The potential of this regeneration depends on the soil depth, ground aspects and animals in the area. On the other hand there is risk of erosion, which is influenced by the slope, the rock type and again the soil depth.
In figure \ref{fig:mp_GIS}
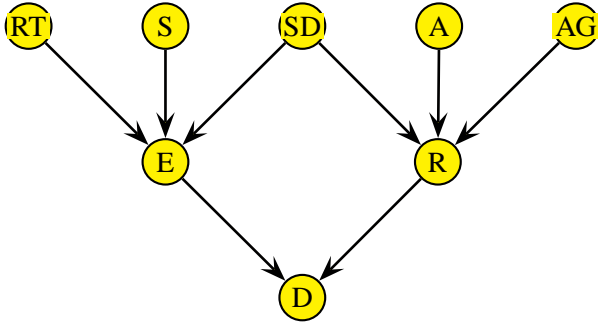
\begin{figure}
  \centering
	\psset{unit=.9cm,arrowscale=2, radius=0.35,fillstyle=solid, fillcolor=yellow}
	\begin{pspicture}[showgrid=false](-4.30,0.00)(4.30,4.30)
	\Cnode(-4,4){node_1} \rput(-4,4){\setlength{\fboxsep}{1pt}\colorbox{yellow}{\color{black} \small{RT}}}
	\Cnode(-2,4){node_2} \rput(-2,4){\setlength{\fboxsep}{1pt}\colorbox{yellow}{\color{black} \small{S}}}
	\Cnode(0,4){node_3} \rput(0,4){\setlength{\fboxsep}{1pt}\colorbox{yellow}{\color{black} \small{SD}}}
	\Cnode(2,4){node_4} \rput(2,4){\setlength{\fboxsep}{1pt}\colorbox{yellow}{\color{black} \small{A}}}
	\Cnode(4,4){node_5} \rput(4,4){\setlength{\fboxsep}{1pt}\colorbox{yellow}{\color{black} \small{AG}}}
	\Cnode(-2,2){node_6} \rput(-2,2){\setlength{\fboxsep}{1pt}\colorbox{yellow}{\color{black} \small{E}}}
	\Cnode(1.9806,2){node_7} \rput(1.9806,2){\setlength{\fboxsep}{1pt}\colorbox{yellow}{\color{black} \small{R}}}
	\Cnode(0,0){node_8} \rput(0,0){\setlength{\fboxsep}{1pt}\colorbox{yellow}{\color{black} \small{D}}}
	\ncline[linewidth=1pt,linecolor=black]{->}{ node_1}{node_6}
	\ncline[linewidth=1pt,linecolor=black]{->}{ node_2}{node_6}
	\ncline[linewidth=1pt,linecolor=black]{->}{ node_3}{node_6}
	\ncline[linewidth=1pt,linecolor=black]{->}{ node_3}{node_7}
	\ncline[linewidth=1pt,linecolor=black]{->}{ node_4}{node_7}
	\ncline[linewidth=1pt,linecolor=black]{->}{ node_5}{node_7}
	\ncline[linewidth=1pt,linecolor=black]{->}{ node_6}{node_8}
	\ncline[linewidth=1pt,linecolor=black]{->}{ node_7}{node_8}
    \end{pspicture}
  \caption{Bayesian Network for Burnt Forest Desertification
  (RT: Rock Type, S: Ground Slope, SD: Soil Depth, A: Ground Aspect, AG: Animal Grazing,
   E: Risk of Erosion, R: Regeneration Potential, D: Risk of Desertification).}
  \label{fig:mp_GIS}
\end{figure}
we show the complete GIS Bayesian Network. We focus on the lower part of this network described by the following nodes: Risk of Erosion $E$, Regeneration Potential $R$ and Risk of Desertification $D$. $E$ and $R$ have three states each, whilst $D$ has five states. Thus the resulting CPT of $D$ will be $5\times9$ matrix.
The paper \citet{BN_NeuralNetworks} details two tables of data, which we will use as input to compute the conditional probability table.
For the convenience of the reader we reproduced the tables and added them to the electronic companion. The first table
describes the training data (39 rows) and the second table gives data (14 rows) for comparative purposes.
Each table contains four main columns: the first one describes the site and the other three represent probability
vectors for each Bayesian node. We have examined the data for any obvious errors. We observed that each probability vector adds up to one as it should. We have also checked whether observations repeat themselves. Inspection of the data shows that some of the combined $E$ and $R$ observations repeat themselves up to five times, which results in 19 distinct combined $E$ and $R$ observations only. The test data $E_t, R_t$ and $Z_t$ for comparative purposes consists out of 14 tuples of which are 11 distinct.

\subsection{Boundary limitation  and potential surge  regressions}
We will now determine the CPT and use it to check its goodness. The background of how to derive the CPT has been explained in previous sections. First we combine parent observations ($X := E \circlearrowleft R$, see algorithm \ref{alg:nNodes2One}). Next we compute the CPT basis by $B^*:=b^*_{i,j} := ((X'X)^{-1} X'D)'$. Finally we transform the CPT basis into a CPT $B$ that guarantees that each element is a probability and that the column sum adds up to one. We use boundary limitations and potential surge methods to achieve the fulfilment of these constraints. The heuristics operate on column 1,2, 7 and 8 and lead to a maximum change of 13.4\% for $b^*_{2,2}$. The boundary limitation method leads to a CPT, that has a lower average shift error 1.17\% than the potential surge method on the test data. This CPT is shown in table \ref{table:GIS-CPT}.
\begin{table*}
    \begin{small}\begin{center}
    \begin{tabular}{r|rrr|rrr|rrr}
       \hline
       $d|er$ & $e_1r_1 $ & $e_1r_2 $ & $e_1r_3 $ & $e_2r_1 $ & $e_2r_2 $ & $e_2r_3 $ & $e_3r_1 $ & $e_3r_2 $ &  $e_3r_3$ \\
        \hline
        $d_1$ &    11.18\% &     1.62\% &    22.58\% &     4.61\% &    14.77\% &    13.09\% &    16.35\% &     0.00\% &     9.30\% \\
        $d_2$ &    34.07\% &     0.00\% &    53.50\% &     1.91\% &    28.70\% &    18.59\% &     5.37\% &    10.86\% &     7.99\% \\
        $d_3$ &    44.70\% &     7.86\% &     8.62\% &    76.22\% &    41.67\% &    53.02\% &    75.84\% &     9.43\% &    55.73\% \\
        $d_4$ &     0.00\% &    71.72\% &     1.73\% &     8.73\% &     6.81\% &    11.84\% &     0.00\% &    47.95\% &    25.47\% \\
        $d_5$ &    10.06\% &    18.80\% &    13.57\% &     8.53\% &     8.05\% &     3.46\% &     2.45\% &    31.76\% &     1.50\% \\
        \hline
    \end{tabular}
    \end{center}\end{small}
    \caption{GIS-CPT derived with probability boundary limitation method.\label{table:GIS-CPT}}
\end{table*}
The average shift error between $n$ parent observations ($E_t,R_t$) multiplied with the CPT $B$ and the test data $D_t$ is computed by:
\begin{equation}
    \epsilon = \frac{1}{m n} \sum_{i,j} |B(E_t \circlearrowleft R_t) - D_t|.
\end{equation}
The CPTs have the same predictive power on the test data, i.e. they select the maximum probability state in 13 out of 14 cases ($\frac{13}{14}=92.9\%$) correctly. The highest average shift error of 1.31 \% occurs when using the distinct data set and the shift-normalisation heuristic. As we can see the difference between the lowest and highest total average shift error of 0.14\% is insignificant. Thus we have demonstrated with this case study that  boundary limitation  and potential surge methods generate CPTs which posses good (92.9\%) predictive power that contain very low (less than 1.31\%) total average shift errors.

\subsection{Logistic Regression}
The multinomial logistic regression gives an average shift error between approximated and observed effect test data of 2.1\%. This was derived from the distinct data set. As a first step we combined the nodes $E$ and $R$ and determined the logistic regression parameters.
Using these it would be possible to compute any effect probabilities. However, in this paper the focus is on the generation of a CPT. Thus it is necessary to determine the CPT based on the logistic regression parameters. This is achieved by applying a $9\times 9$ canonical matrix $E$ as input to the equation system \ref{eq:LL-CPT}. This results in the CPT $\tilde{C}$ shown in table \ref{table:GIS-CPT-MLR}.
\begin{table*}
    \begin{small}\begin{center}
        \begin{tabular}{r|rrr|rrr|rrr}
        \hline
           $d|er$ & $e_1r_1 $ & $e_1r_2 $ & $e_1r_3 $ & $e_2r_1 $ & $e_2r_2 $ & $e_2r_3 $ & $e_3r_1 $ & $e_3r_2 $ & $e_3r_3 $ \\
        \hline
            $d_1$ &     13.9\% &      0.1\% &     23.3\% &      4.5\% &     23.0\% &     10.6\% &     14.1\% &      1.8\% &      8.9\% \\
            $d_2$ &     28.4\% &      0.4\% &     50.6\% &      7.6\% &     22.0\% &     19.5\% &      8.2\% &      5.2\% &      9.4\% \\
            $d_3$ &     47.6\% &      0.8\% &     10.9\% &     70.7\% &     40.4\% &     55.5\% &     70.2\% &     10.6\% &     51.2\% \\
            $d_4$ &      2.3\% &     96.9\% &      4.5\% &      7.8\% &      7.3\% &     10.4\% &      5.9\% &     28.0\% &     26.6\% \\
            $d_5$ &      7.9\% &      1.8\% &     10.7\% &      9.4\% &      7.4\% &      4.0\% &      1.7\% &     54.4\% &      3.9\% \\
        \hline
        \end{tabular}
    \end{center}\end{small}
    \caption{GIS-CPT $\tilde{C}$ derived with Logistic Multinomial Regression.\label{table:GIS-CPT-MLR}}
\end{table*}
We use this CPT in order to determine the effect probabilities for the test data by $\tilde{Z} = \tilde{C} (E_t\circlearrowleft R_t$). This must be compared with the effects from the test data $Z_t$. In figure \ref{fig:MPGIS-LogitCPT-errors-test-data} we show several comparisons.
\begin{figure*}
  \centering
  \includegraphics[height=7cm]{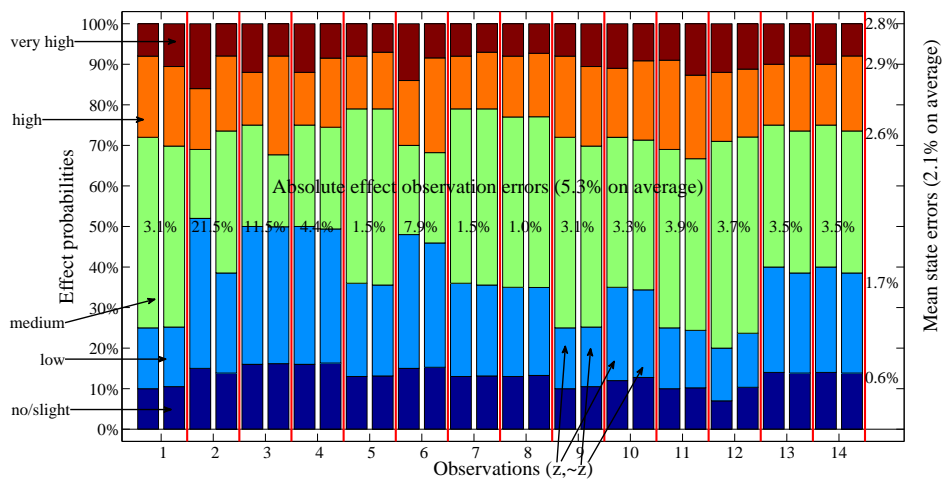}
  \caption{Logistic Regression - Burnt Forest Desertification: Goodness of computed effect probabilities by comparing them to the given test data.}
  \label{fig:MPGIS-LogitCPT-errors-test-data}
\end{figure*}
We have placed each test data effect next to an approximated effect. Thus we have compared 14 effects and their approximation with each other. The absolute errors of each pair are given. We notice that effect observation two and three lead to approximations with deviations of 22.5\% and 11.6\% respectively. Such deviations indicate irregularity and suggest a verification of the test data. Hence, this indicates that our methods can be used to identify data irregularities. On average absolute effect observation errors are 5.3\%, which appears to be quite high. That means if we require that all states are correctly identified we will be wrong in one out of 20 cases roughly spoken. On the other hand if we are only interested in identifying one state correctly the average state error is 2.1\%.
Notice that the first state (no/slight) is approximated best with 0.6\%. The worst state to approximate is the ``high" state, which has a mean state error of 2.9\%. 13 out of 14 sample states were chosen correctly, i.e. 92.9\% diagnostic goodness. If we compare the potential surge method and the logistic regression total average shift errors with each other we see that $1.3\%-2.1\% = -0.8\%$ the logistic regression is performing worse.

The comparison of the effect training data and the approximations,
show that the total average shift error for the test data is 1.4\%. In case we use the distinct training data we get the following error measures. The diagnostic error is 7.1\%. The absolute effect observation error is on average 5.3\%. The total average shift error is 2.1\%.

Using the potential surge method gave a diagnostic value of 92.9\% and a total average shift error of 1.3\%. The logistic regression had the same diagnostic value but a slightly higher total average shift error of 1.4\%. Thus we can assume that both CPT approximations are fairly reliable.


\section{Case Study (2) - Military Effects}\label{sec:PPD_Real_World_Application}
 The case study  presented here contains military effects, which constitute critical success factors (CSF) for a vignette. Typical military effects are find, destroy, secure and many more (see \citet{Louvieris_2007} for details). These effects are assessed by means of a Bayesian Network.

\subsection{Bayesian Network and Data}
 The introduced effect models depend on the status of the situation and the relative morale. In this paper we will not discuss the challenges of obtaining evidence for these measures nor the derivation of this Bayesian Network - respective information can be found in \citet{DBLP:conf/damas/LouvierisGMWOLH05}.
 We will focus our discussion on converging Bayesian Networks.
The effect CSF node $E$ has two parents: situation $S$ and relative morale $M$.

 Each node has three states. The objective is to derive the CPT for $E$.
Two lots of four subject matter experts (SME) were asked to provide the CPT directly, i.e. each group had to do a consensus labelling.
Additionally they had to provide probability vectors for $S, M$ and $E$ for the effects observed in the vignettes, whenever they noticed changes. The consensus labelling produced a table of 83 rows and 9 columns (see table in electronic companion).
Deleting identical entries gives 54 distinct sample points (rows).

Subject matter experts provided the CPT show in table \ref{table:SME_CPT}.
\begin{table*}
  \centering
	\begin{tabular}{r|rrr|rrr|rrr}
\hline
 Situation &           \multicolumn{ 3}{|c}{good} &       \multicolumn{ 3}{|c}{critical} &            \multicolumn{ 3}{|c}{bad} \\
    Morale &       high &     parity &        low &       high &     parity &        low &       high &     parity &        low \\
\hline
  achieved &     95.0\% &     85.0\% &     70.0\% &     50.0\% &     50.0\% &     25.0\% &      5.0\% &      3.0\% &      2.0\% \\
   at risk &      3.0\% &     11.0\% &     20.0\% &     35.0\% &     30.0\% &     35.0\% &      5.0\% &      4.0\% &      3.0\% \\
not achieved &      2.0\% &      4.0\% &     10.0\% &     15.0\% &     20.0\% &     40.0\% &     90.0\% &     93.0\% &     95.0\% \\
\hline
\end{tabular}
  \caption{Subject matter expert CPT.}
  \label{table:SME_CPT}
\end{table*}
We noted that certain probability entries varied from one SME to another by up to 25\% before agreement. A direct comparison of the SME CPT and the Regression CPT with probability potential surge method shows significant differences. This indicates that SMEs have substantially different interpretations of factors.

\subsection{Probability boundary limitation and potential surge method}
Using the probability boundary limitation method leads to an absolute column sum error of 88.4\% and an average shift error of 29.5\%.
When the regression CPT was applied to the observations the correct effect state was chosen in 38 out of 54 cases (70.4\%). Here ``correct" is based on the assumption that the CPT provided by the SME is without any fault.
If we assume the effect observations are the basis for comparison, we obtain that the correct effect state with $B$ is chosen 50 times (i.e. diagnostic goodness of 92.6\%), whilst the usage of the SME CPT lead to 40 correct state selections (i.e. diagnostic power of 74.0\%).

For the command and control agents the correct choice of the state is important so that agents can take the appropriate actions. The above results are encouraging because given a similar scenario agents will make right decisions in the majority of cases (this case study suggest 93\%). On the other hand without prior training data (i.e. SME CPT only) the right choice would have been done in 74\% of the observation points. Thus the generated CPT has a 19\% better diagnostic value.

When we apply the probability potential surge method we can improve the correct state choice up to 98.2\% (only one state was chosen incorrectly) under the assumption that effect states were correct. Table \ref{table:Regression_CPT_shift_quick_fix} shows the regression CPT with the \emph{potential surge} method.
\begin{table}[!h]
  \centering
\begin{tabular}{rr|rrr}
\hline
\begin{sideways}Situation\end{sideways} & \multicolumn{1}{c|}{\begin{sideways}Morale\end{sideways}} & \multicolumn{1}{c}{\begin{sideways}achieved\end{sideways}} & \multicolumn{1}{c}{\begin{sideways}at risk\end{sideways}} & \multicolumn{1}{c}{\begin{sideways}not achieved\end{sideways}} \\
\hline
\multicolumn{1}{c}{\multirow{3}[2]{*}{\begin{sideways}good\end{sideways}}} & high  & 49.1\% & 50.9\% & 0.0\%\\
\multicolumn{1}{c}{} & parity & 84.2\% & 0.0\% & 15.8\% \\
\multicolumn{1}{c}{} & low   & 0.0\% & 59.4\% & 40.6\%\\
\hline
\multicolumn{1}{c}{\multirow{3}[2]{*}{\begin{sideways}critical\end{sideways}}} & high  & 0.0\% & 42.3\% & 57.7\%\\
\multicolumn{1}{c}{} & parity & 0.0\% & 60.3\% & 39.7\% \\
\multicolumn{1}{c}{} & low   & 60.8\% & 0.0\% & 39.2\%\\
\hline
\multicolumn{1}{c}{\multirow{3}[2]{*}{\begin{sideways}bad\end{sideways}}} & high  & 70.0\% & 0.0\% & 30.0\%\\
\multicolumn{1}{c}{} & parity & 0.0\% & 20.8\% & 79.2\% \\
\multicolumn{1}{c}{} & low   & 23.8\% & 30.0\% & 46.1\%\\
\hline
\end{tabular}%

  \caption{Regression CPT with potential surge method.}
  \label{table:Regression_CPT_shift_quick_fix}
\end{table}

Thus we have shown in this case study that the approximated CPT has a high diagnostic value (98.2\%) of correct decisions, whilst the SME CPT managed 74\% correct decision - assuming correct effect probabilities. This demonstrated that a generated CPT from the cause and effect observations will have a better diagnostic power (by 24.2\%) than the SME CPT.

\section{Conclusion}

A review on the existing CPT generation methods revealed the need to apply multivariate data analysis techniques to Bayesian Networks rather than noisy-functional methods.
Thus, the main contributions of this paper are new CPT generation methods in particular when the input from parent nodes is observed as soft evidence. The first two methods - the boundary limitation  and potential surge  method - map conditional expectations to the probability space.
The third - probability base vector extraction - method is manifested through its application to the multinomial logistic regression.
Introducing Bayesian Networks using matrix notation enabled us to propose tools such as a combine operator and the necessary ``column extraction" proposition. The operator combines parent probability vectors of a converging Bayesian Network into a single probability vector. This allows the efficient computation of effect probabilities and assists in the generation of CPTs.  The ``column extraction" proposition proves particular beneficial when using non-linear methods to obtain a representative CPT.

The CPT generation methods proposed in this paper were applied to two case studies and demonstrated high predictive power. The proposed measures of goodness give clear and intuitive statements about the quality of the approximations and are found to be the best measurement methods in this context.
The military effects case study showed that generic subject matter expert CPTs do not achieve the quality of the generated CPTs in respect to the effect probability computations. The burnt forest case study showed the goodness of the developed multivariate data CPT generation techniques. A comparison of the methods in respect to the case studies suggest similar goodness of the techniques. Moreover, the potential surge  and boundary limitation techniques demonstrated to be better than probability base vector extraction method applied to multinomial logistic regression.

There are several avenues for future research. The first is to gather further evidence of the effectiveness of the proposed methods by applying them to the classic Bayesian Network test instances\footnote{\url{http://www.bnlearn.com/bnrepository/}}. However, this would require the generation of hard and soft evidence for the nodes. Furthermore, this should also be accompanied by a theoretical study.  
The second avenue of research should be an analysis of the here introduced CPT comparison measure versus the K-L and euclidean measure. As a third avenue it should be investigated of whether the converging Bayesian Networks are superior to the Extended Belief Rule-Based systems \cite{Calzada2015}.

In summary the here proposed CPT generation methods will lead to more reliable decision making within Bayesian Networks when it comes to the prediction of effects.

 \section*{Acknowledgement}
  The work reported in this paper was sponsored by the United Kingdom MOD Data and Information Fusion Defence Technology Centres (DIF DTC) Research Programme.

  The ``burnt forest paper" \citet{BN_NeuralNetworks} proofed to be a useful resource.
  We would like to thank Maria Petrou for providing feedback.


\bibliographystyle{wg-elsarticle-harv}
\bibliography{doc_CPT-2015}


\end{document}